\documentclass[sigconf,nonacm]{acmart}

\AtBeginDocument{%
  }

\usepackage{booktabs}
\usepackage{multirow}
\usepackage{graphicx}
\graphicspath{{figures/}{./}}
\usepackage{amsmath,amsthm}
\usepackage{nicefrac}
\usepackage{float}
\usepackage[ruled,vlined,linesnumbered]{algorithm2e}
\newtheorem{proposition}{Proposition}

\renewcommand\footnotetextcopyrightpermission[1]{}

\begin{document}

\title{CRISP: Scalable Importance-Stratified Coresets for Imbalanced Tabular Learning}
\titlenote{The views expressed in this article are those of the authors alone and do not necessarily reflect the views of Coinbase or its affiliates.}

\author{Hardhik Mohanty}
\affiliation{%
  \institution{University of Southern California}
  \city{Los Angeles}
  \state{California}
  \country{USA}}
\email{hmohanty@usc.edu}

\author{Indrayana Rustandi}
\affiliation{%
  \institution{Coinbase}
  \city{San Francisco}
  \state{California}
  \country{USA}}
\email{indrayana.rustandi@coinbase.com}

\author{Mohamadreza Sheibani}
\affiliation{%
  \institution{Coinbase}
  \city{San Francisco}
  \state{California}
  \country{USA}}
\email{mohamadreza.sheibani@coinbase.com}

\renewcommand{\shortauthors}{Mohanty, Rustandi, and Sheibani}

\begin{abstract}
Large imbalanced tabular datasets make repeated gradient-boosted tree training expensive. Existing coreset methods often lose accuracy when most majority examples are removed. We present \textbf{CRISP} (\textbf{C}oreset \textbf{R}eduction via \textbf{I}mportance-\textbf{S}tratified \textbf{P}runing), a linear-time method that allocates a negative-class budget across quantile strata of a proxy-model score. Sample weights account for unequal inclusion probabilities. At $95\%$ negative-class reduction on a production fraud dataset, CRISP trains on approximately $1.70$M of $25$M rows and retains $99.7\%$ of full-data Average Precision. This is a $93.2\%$ reduction in total training rows. On public CriteoPrivateAds, CRISP has the highest mean Average Precision at each tested rate from $90\%$ to $99.4\%$ majority reduction. Sparkov results are mixed at lower rates, but CRISP has the highest mean at $99.2\%$ and $99.4\%$. Ablations identify budget allocation and inverse-propensity weighting as the main sources of the production-dataset gain.
\end{abstract}



\maketitle

\section{Introduction}

Large tabular datasets create a recurring cost rather than a one-time training expense. A production model is retrained as labels mature, transaction patterns change, and new features become available. Each candidate model may also be trained many times during validation, calibration, and hyperparameter search. U.S.\ ACH volume rose from $11.6$ billion transactions in 2014 to $21.6$ billion in 2025~\cite{fedach}, illustrating the growth faced by financial systems. Similar scale appears in advertising and commerce. Gradient-boosted decision trees (GBDTs) such as XGBoost~\cite{xgboost} and LightGBM~\cite{lightgbm} remain strong choices for these data, but repeated training on tens of millions of rows can dominate the development cycle.

The cost is not distributed evenly across examples. In fraud and conversion prediction, most rows belong to the negative class and are easy to classify. Many occupy dense interior regions that add little new information once the model has seen enough similar examples. The rare positive class and the smaller set of boundary negatives carry much more of the signal needed to distinguish classes. Uniformly discarding rows ignores this difference. Uniform sampling is attractive because it is simple and preserves the population distribution in expectation, yet it removes informative and redundant negatives at the same rate.

A coreset attempts to retain the information needed for training in a smaller weighted subset~\cite{coresetsurvey}. We focus on one-shot selection, where the subset is created before target-model training and reused across later runs. Reuse matters in practice. A selector that costs as much as the full hyperparameter sweep offers little benefit, while a selector that can be amortized across models may save substantial computation. This requirement favors methods with one lightweight proxy-model fit and a small number of passes over the data.

Class imbalance places a strong constraint on subset selection for fraud and conversion data. Removing positives is risky when they represent less than $2\%$ of the data~\cite{eurocard,deepfraud}. Our pipeline therefore retains all positive examples and applies the reduction budget to negatives only. This distinction affects how reduction is reported. A $95\%$ negative-class reduction is not a $95\%$ reduction in all rows because the positive class remains intact. We report both the selector's negative-class rate and the corresponding target-training size.

Existing selection strategies make different tradeoffs between coverage, model awareness, and cost. Geometry-based methods preserve feature-space coverage but can become expensive or unreliable in high dimensions. Gradient-matching methods use a more direct learning objective, although pairwise similarities and iterative updates limit their use at production scale~\cite{craig,gradmatch}. Training-dynamics methods identify easy, ambiguous, and hard regions from predictions collected across checkpoints~\cite{coretab,datamaps}. They are informative but require repeated inference. Coverage-Centric Coreset Selection (CCS) avoids concentrating the subset on only the hardest examples by sampling across score strata~\cite{ccs}. Its uniform allocation, however, gives the same budget to strata whose estimated learning value may differ sharply.

Direct importance sampling addresses allocation but can lose coverage. At high reduction rates, most of the selected rows may come from a narrow upper tail of the score distribution. Zheng et al.\ report that this behavior can perform worse than random sampling~\cite{ccs}. Unequal sampling also changes the effective training distribution unless the target learner receives appropriate weights. These two issues motivate the design of CRISP: preserve broad score coverage through quantile strata, then spend more of the budget on strata with larger proxy-model scores and correct unequal inclusion with inverse-propensity weights.

CRISP trains a lightweight proxy GBDT whose sole role is to assign importance scores to negative rows before target-model training. The proxy model is separate from the downstream model and is used only during coreset construction. Equal-count score quantiles prevent a skewed score distribution from collapsing into one dominant bin. The mean score of each stratum determines its share of the negative budget. Sampling then occurs within each stratum, and selected negatives receive clipped inverse inclusion weights. Only the fixed-size proxy training sample is collected centrally. Scoring, quantiles, aggregation, and sampling remain distributed, giving linear cost in the number of training rows.

We evaluate this design on three datasets with different scales and domains. The proprietary production dataset contains about $25$ million rows, of which $1.9\%$ are positive examples. This proportion describes the constructed experimental training dataset and is not an overall or population-level payment-reversal rate. CriteoPrivateAds provides a public advertising benchmark with roughly $85$ million training rows~\cite{criteoprivatead}. Sparkov~\cite{sparkovdataset} supplies a smaller synthetic fraud benchmark where additional baselines are feasible. On the production dataset, CRISP retains $99.7\%$ of full-data monthly AP after removing $95\%$ of negatives, while using approximately $6.8\%$ of the original rows. On CriteoPrivateAds, CRISP has the highest mean AP at every tested rate. Sparkov is less uniform and exposes the sensitivity of five-run averages to outliers.

The contributions are:
\begin{enumerate}
    \item A linear-time, importance-allocated stratified sampling framework for highly imbalanced tabular data, with explicit accounting for the positive-class retention floor.
    \item An estimator analysis that separates the unbiased, unclipped inverse-propensity estimator from the biased clipped estimator used for stable GBDT training.
    \item Production-scale evidence at up to $95\%$ negative-class reduction, public cross-domain evidence from CriteoPrivateAds, and component ablations covering stratification, score construction, clipping, and proxy-model capacity.
    \item Row-count tradeoff plots and public benchmark configurations that distinguish target-training size from end-to-end selection cost.
\end{enumerate}

CRISP is designed for high reduction rates, where a small negative budget makes allocation and coverage especially important. The experiments emphasize this regime while retaining lower-rate results where available. Section~\ref{sec:related} reviews coreset selection, training-dynamics methods, and unequal-probability sampling. Section~\ref{sec:method} presents CRISP and its distributed implementation, followed by the estimator properties in Section~\ref{sec:analysis}. Section~\ref{sec:experiments} describes the datasets, baselines, ablations, and cross-dataset results. Sections~\ref{sec:reproducibility} and~\ref{sec:limitations} discuss reproducibility, limitations, and ethical considerations, and Section~\ref{sec:conclusion} concludes the paper.

\section{Related Work}
\label{sec:related}

\textbf{Imbalanced instance selection.} Early instance-selection methods remove majority examples with local geometric rules. Condensed Nearest Neighbors retains a small set that preserves nearest-neighbor decisions, while Tomek Links removes overlapping pairs near a class boundary~\cite{cnn,tomek}. NearMiss instead keeps majority points close to minority examples~\cite{nearmiss}. These methods are intuitive on small metric datasets, but distance computations become expensive and less informative as dimension and sample count grow. LSH-based selection reduces this cost by replacing exact neighborhoods with hash-based approximations~\cite{lsh_imbalance}. Its selection criterion is still independent of the model eventually trained on the subset.

\textbf{Proxy and geometry-based coresets.} Active-learning coresets use feature-space coverage to select points that are far from the current labeled set~\cite{coresetactive}. Selection via Proxy reduces the cost of repeated ranking by using a smaller proxy model to score examples~\cite{selectionproxy}. FCTR combines gradient approximations with clustering and product quantization, reporting a three- to ten-fold selection speedup over gradient-based competitors~\cite{fctr}. More recent approaches reconstruct a learned decision boundary~\cite{mindboundary} or optimize the smallest subset that satisfies a model-performance constraint~\cite{refinedcoreset}. These methods strengthen geometric or optimization objectives, but most were evaluated outside large-scale tabular settings and at moderate pruning rates.

\textbf{Training dynamics and sample importance.} Example forgetting counts how often a model changes an example from correct to incorrect during training~\cite{exampleforgetting}. Dataset Cartography summarizes confidence and variability across checkpoints to identify easy, ambiguous, and hard regions~\cite{datamaps}. GraNd and EL2N estimate importance from early loss gradients or prediction errors~\cite{datadiet}. TracIn measures influence through gradients saved at several training checkpoints~\cite{tracin}. SAMIS uses the discrepancy between sharpness-aware and standard optimization as an estimate of memorization~\cite{samis}. These scores can identify useful or atypical examples, although collecting trajectories or training multiple models can cost as much as the downstream training that pruning is meant to avoid.

\textbf{Gradient and validation matching.} CRAIG selects a subset whose aggregate gradient approximates the full-data gradient~\cite{craig}. GradMatch directly minimizes gradient mismatch, and GLISTER optimizes a validation objective~\cite{gradmatch,glister}. These methods give a clearer optimization target than scalar difficulty scores. Their pairwise similarities, repeated gradient evaluations, or iterative subset updates limit their use on tens of millions of rows. A comparison on a small downsample can still be informative, but it does not answer whether the selector can run on the production-scale dataset itself.

\textbf{High reduction and tabular models.} CCS addresses a failure that appears when a score-based selector keeps only the hardest examples. It divides examples into difficulty strata and allocates budget uniformly so that easy and moderate examples remain represented~\cite{ccs}. Uniform allocation protects coverage, but it does not distinguish between strata when the final budget is very small. CoreTab adapts training-dynamics ideas to GBDTs by building datamaps from predictions at several boosting checkpoints~\cite{coretab}. This is one of the few coreset methods designed for modern tabular models, though its repeated prediction passes are costly on the largest datasets.

\textbf{Connection to survey sampling.} Unequal-probability sampling has long used inverse inclusion weights to estimate a finite-population total~\cite{horvitzthompson}. CRISP follows this principle after allocating a negative-class budget across proxy-score strata. It differs from CCS in the allocation rule and from direct importance sampling in its use of quantile coverage. It also differs from gradient-matching coresets because it never forms pairwise gradient similarities. The practical gap addressed here is narrower than general coreset construction: a reusable subset for imbalanced tabular GBDTs, selected with linear passes at negative-class reduction rates above $90\%$. Weight clipping means the deployed estimator is not exactly unbiased, so Section~\ref{sec:analysis} states the resulting bias rather than claiming a general variance guarantee.

\section{Methodology: CRISP}
\label{sec:method}

\subsection{Problem Formulation}
Let $\mathcal{V}=\{(x_i,y_i)\}_{i=1}^{n}$ with $y_i\in\{0,1\}$. The positive and negative index sets are $\mathcal{C}_+$ and $\mathcal{C}_-$, with sizes $n_+$ and $n_-$. CRISP retains every positive row and reduces only the negative class. For a negative-class reduction rate $r_-$, the target negative budget is
\begin{equation}
    k_-=\left\lfloor(1-r_-)n_-\right\rfloor.
    \label{eq:negative_budget}
\end{equation}
The expected training size is therefore $n_+ + k_-$ rather than $(1-r_-)n$. This distinction matters at high reduction rates because positives become a large share of the retained rows.

\subsection{Budget Interpretation}
Negative-class reduction and total-row reduction are related but not identical. If the negative budget is met exactly, the fraction of all rows removed is
\begin{equation}
    r_{\mathrm{all}}=r_-\frac{n_-}{n}.
    \label{eq:total_reduction}
\end{equation}
The retained positive fraction places a floor on target-training size. On the production dataset, $95\%$ negative reduction removes about $93.2\%$ of all rows because positives account for roughly $1.9\%$ of the experimental training table. The difference is larger on Sparkov and Criteo at rates above $99\%$, where retained positives can outnumber retained negatives.

\begin{table}[t]
\centering
\caption{Training rows implied by the retain-all-positives rule.}
\label{tab:derived_rows}
\small
\setlength{\tabcolsep}{4pt}
\begin{tabular}{@{}lrrrr@{}}
\toprule
\textbf{Dataset} & \textbf{Full} & \textbf{90\%} & \textbf{95\%} & \textbf{99.4\%} \\
\midrule
Production & $25.0$M & $2.923$M & $1.697$M & $0.617$M \\
Sparkov & $1.297$M & $136{,}423$ & $71{,}964$ & $15{,}241$ \\
CriteoPrivateAds & $85.0$M & $9.089$M & $4.872$M & $1.161$M \\
\bottomrule
\end{tabular}
\end{table}

We use $r_-$ throughout the paper because it is the parameter passed to every selector. Figures that discuss training cost use the derived total rows instead. This keeps method comparisons at the same negative budget while making the actual target-model input size visible.

\subsection{Proxy Model and Importance Scores}
A lightweight proxy GBDT is trained on at most $N_{\mathrm{sub}}$ negatives and all positives. This model is used only to rank examples for selection and is not reused as the downstream classifier. Class weighting is used during its fit. Its output $p_i$ is a ranking score for fraud risk, not a calibrated estimate of the population fraud rate. For a negative row, CRISP computes
\begin{equation}
    s_i=\max\{\epsilon,\alpha p_i+\beta p_i(1-p_i)\},
    \qquad i\in\mathcal{C}_-,
    \label{eq:importance}
\end{equation}
where $\epsilon>0$ prevents zero importance scores. A strictly positive inclusion probability also requires the rounded stratum allocation $k_q$ to be positive. The experiments use $\alpha=\beta=1$. In that setting $s(p)=2p-p^2$ is monotone for $p\in[0,1]$. The proxy-model ranking is therefore driven by $p_i$, while the nonlinear transformation changes the relative budget assigned to score strata.

\subsection{Budget Allocation and Sampling}
The negative rows are partitioned into $Q$ equal-count quantile strata $B_1,\ldots,B_Q$. Let $\bar{s}_q$ be the mean score in stratum $q$. CRISP assigns
\begin{equation}
    k_q \propto \bar{s}_q,
    \qquad \sum_{q=1}^{Q} k_q=k_-,
    \label{eq:allocation}
\end{equation}
with integer rounding and redistribution when $k_q>|B_q|$. Within a stratum, define $u_i=s_i^\gamma$. The inclusion probability used by the evaluated Bernoulli sampler is
\begin{equation}
    \pi_i=\min\!\left(1,\frac{k_q u_i}{\sum_{j\in B_q}u_j}\right).
    \label{eq:inclusion_probability}
\end{equation}
The production implementation uses $\gamma=1$, which favors larger scores inside each stratum. Sparkov and CriteoPrivateAds use $\gamma=0$, which samples uniformly inside each stratum. All variants share the same quantile construction and importance-based allocation. Since sampling is Bernoulli, $k_q$ is an expected budget before probability saturation and the realized coreset size varies by seed.

Selected negatives receive a clipped inverse inclusion weight
\begin{equation}
    \widetilde w_i=\min\{1/\pi_i,w_{\max}\}.
    \label{eq:clipped_weight}
\end{equation}
All positives receive weight one. The target GBDT consumes these values as sample weights.

\begin{algorithm}[t]
\DontPrintSemicolon
\small
\caption{\textbf{CRISP}: Coreset Selection}\label{alg:crisp}
\KwIn{Dataset $\mathcal{V}$, negative reduction $r_-$, strata $Q$, score weights $\alpha,\beta$, within-stratum exponent $\gamma$}
\KwOut{Coreset $\mathcal{S}$ with weights $\{w_i\}$}
$\mathcal{C}_+,\mathcal{C}_-\leftarrow\textsc{Partition}(\mathcal{V})$;\;
$k_-\leftarrow\lfloor(1-r_-)\cdot|\mathcal{C}_-|\rfloor$\;
$f_{\mathrm{proxy}}\leftarrow\textsc{TrainProxy}(\text{subsample of }\mathcal{V})$\;
\ForEach{$i\in\mathcal{C}_-$}{
  $p_i\leftarrow f_{\mathrm{proxy}}(x_i)$;\;
  $s_i\leftarrow\max\{\epsilon,\alpha p_i+\beta p_i(1-p_i)\}$\;
}
$\{B_q\}\leftarrow\textsc{QuantilePartition}(\{s_i\},Q)$\;
\ForEach{stratum $B_q$}{
  $k_q\leftarrow k_-\cdot\bar{s}_q/\sum_{q'}\bar{s}_{q'}$\;
  \ForEach{$i\in B_q$}{
    $\pi_i\leftarrow\min(1,k_q s_i^\gamma/\sum_{j\in B_q}s_j^\gamma)$\;
    Sample $i$ w.p.\ $\pi_i$; set $w_i\leftarrow\min(1/\pi_i,w_{\max})$\;
  }
}
\Return{$\mathcal{C}_+\cup\textsc{SelectedNegatives},\{w_i\}$}\;
\end{algorithm}

\subsection{Distributed Implementation}
Figure~\ref{fig:architecture} summarizes the selection pipeline. The selector keeps the full negative table in Spark. A fixed-size negative sample and all positives are transferred to the driver for proxy-model fitting. The fitted model is then broadcast to the workers, which score negative rows in vectorized batches. Approximate quantiles form the stratum boundaries. A grouped aggregation returns each stratum's row count, mean score, and score sum. Since this summary contains only $Q$ rows, the allocation table can be broadcast back to the workers before sampling.

This design avoids collecting the full score vector on one machine. It also makes the distinction between target and realized size explicit. The allocation routine produces the target $k_q$ for each stratum, while Bernoulli sampling determines the count observed in a given run. The implementation records both the target and realized negative counts, along with the range of sample weights. These diagnostics are especially important when the positive class forms a substantial fraction of the final training set.

\textbf{Complexity.} Proxy-model training uses a sample capped independently of $n$. Scoring, approximate quantiles, aggregation, and sampling require a small constant number of distributed passes. The overall selection cost is therefore $\mathcal{O}(n)$ in the number of training rows, with memory distributed across the Spark workers.

\begin{figure*}[t]
\centering
\includegraphics[width=0.8\textwidth]{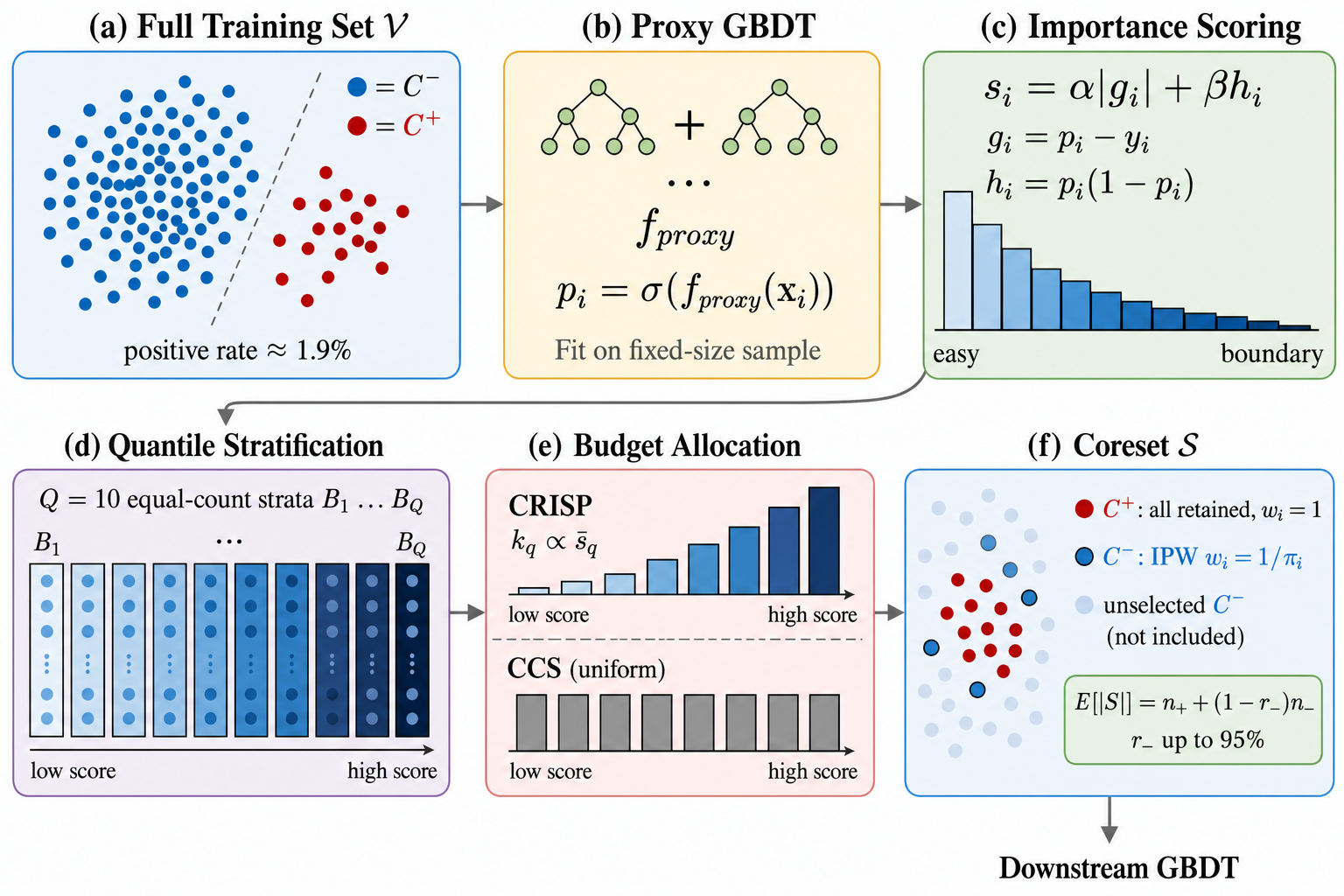}
\Description{CRISP keeps all positive rows, scores negative rows with a proxy model, forms score quantiles, allocates the negative budget by mean score, samples within each stratum, and applies clipped inverse-propensity weights before target-model training.}
\caption{CRISP selection pipeline. The input rate $r_-$ reduces the negative class only. Quantile strata preserve score coverage, while their mean scores determine budget. The realized negative count is random under Bernoulli sampling.}
\label{fig:architecture}
\end{figure*}

\section{Design Rationale and Estimator Properties}
\label{sec:analysis}

\subsection{Coverage Under a Small Budget}
Proxy-model scores on imbalanced data are typically concentrated near zero with a much smaller upper tail. Equal-width bins can therefore place most negatives into one stratum and leave several bins nearly empty. Equal-count quantiles avoid that collapse. Every stratum starts with a comparable number of candidates, and the allocation step decides how many to retain from each part of the ranking.

This construction also separates two decisions that direct importance sampling combines. The strata determine which score regions remain represented. The allocation determines how much resolution each region receives. A high-score stratum may receive most of the budget, but low and moderate strata retain nonzero inclusion probabilities. The $Q=1$ ablation removes this separation and recovers the direct importance-sampling behavior discussed by Zheng et al.~\cite{ccs}.

\subsection{Weighting and Estimator Scope}
The weighting argument follows unequal-probability survey sampling~\cite{horvitzthompson}. Fix the dataset $\mathcal V$ and model parameter $\theta$. For $\ell_i(\theta)=\ell(f_\theta(x_i),y_i)$, define the full-data empirical risk
\begin{equation}
    L_{\mathcal V}(\theta)=\frac{1}{n}\sum_{i=1}^{n}\ell_i(\theta).
    \label{eq:empirical_risk}
\end{equation}
Let $\mathcal D$ denote the realized sampling design, including the fitted proxy model, scores, strata, rounded allocations, and resulting inclusion probabilities.

\begin{proposition}[Conditional unbiasedness of unclipped IPW]
\label{prop:unbiased}
Let $I_i$ indicate whether negative row $i$ is selected. Assume $\pi_i\in(0,1]$ almost surely and $\mathbb{E}[I_i\mid\mathcal D]=\pi_i$ for each $i\in\mathcal C_-$. If selected negatives use $w_i=1/\pi_i$, then
\begin{equation}
    \widehat L(\theta)=\frac{1}{n}
    \left(\sum_{i\in\mathcal{C}_+}\ell_i(\theta)
    +\sum_{i\in\mathcal{C}_-}I_i\frac{\ell_i(\theta)}{\pi_i}\right)
\end{equation}
satisfies
\begin{equation}
    \mathbb{E}[\widehat L(\theta)\mid\mathcal D]
    =L_{\mathcal V}(\theta),
\end{equation}
and therefore $\mathbb{E}[\widehat L(\theta)]=L_{\mathcal V}(\theta)$.
\end{proposition}

The result uses only first-order inclusion probabilities and does not require independent selections. The deployed estimator clips each negative weight at $w_{\max}$. For a nonnegative loss, its conditional downward bias is
\begin{equation}
    L_{\mathcal V}(\theta)
    -\mathbb{E}[\widehat L_{\mathrm{clip}}(\theta)\mid\mathcal D]
    =\frac{1}{n}\sum_{i:\pi_i<1/w_{\max}}
    \left(1-\pi_i w_{\max}\right)\ell_i(\theta).
    \label{eq:clipping_bias}
\end{equation}
Clipping is therefore a stability choice rather than an unbiased correction. The ablation in Section~\ref{sec:ablation} measures its empirical effect.

Proposition~\ref{prop:unbiased} concerns empirical loss at a fixed $\theta$. It does not imply that training on the weighted sample produces the same model, nor does it guarantee AP retention after optimization. CRISP is therefore a task-aware subset selector rather than a strong coreset that approximates every parameter setting. The empirical evaluation tests whether the weaker estimator property is useful for downstream GBDT training.

\subsection{Allocation as a Practical Heuristic}
Classical stratified sampling allocates more rows to strata with larger within-stratum variation. The optimal allocation depends on quantities that are unavailable before target training. CRISP substitutes the mean proxy-model score for the unknown variation. This resembles Neyman allocation when score and loss variation are correlated, but it is not an estimate of the exact Neyman solution.

The design makes a testable prediction. Importance-aware allocation should matter little when the budget is large and should separate from uniform allocation as $k_-$ shrinks. The production dataset and Criteo follow that pattern. Sparkov is less consistent, which is evidence against treating the allocation rule as universally optimal. The uniform-allocation baseline, the $Q=1$ ablation, and the no-weight ablation test different parts of this prediction.

\section{Experiments}
\label{sec:experiments}

Table~\ref{tab:datasets} lists the three evaluation datasets. The production dataset contains financial transactions. Sparkov is synthetic fraud data, and CriteoPrivateAds is a public advertising conversion dataset.

\begin{table}[t]
\centering
\caption{Training-set statistics. Criteo positives are conversions. Production and Criteo counts are approximate; production counts describe the constructed experimental training dataset rather than a population-level rate.}
\label{tab:datasets}
\small
\begin{tabular}{@{}lrrrl@{}}
\toprule
\textbf{Dataset} & \textbf{Rows} & \textbf{Positive} & \textbf{Feat.} & \textbf{Source} \\
\midrule
Production & ${\sim}25$M & ${\sim}470$K & 125 & Proprietary \\
Sparkov & $1{,}296{,}675$ & $7{,}506$ & 14 & Public, synthetic \\
CriteoPrivateAds & ${\sim}85$M & ${\sim}655$K & 51 & Public \\
\bottomrule
\end{tabular}
\end{table}

\subsection{Datasets and Protocol}
\textbf{Production dataset.} The production task predicts payment reversal fraud. Training covers April through September 2025, followed by temporal calibration and validation windows through January 2026. The reported metrics are Average Precision on monthly and weekly temporal holds, written AP$_{\mathrm{month}}$ and AP$_{\mathrm{week}}$. Each selected dataset receives a FLAML zero-shot LightGBM configuration~\cite{flaml} and follows the same calibration procedure. Because the zero-shot suggestion is computed from the selected data, target hyperparameters can differ across methods. We state this source of variation rather than treating the comparison as fixed-hyperparameter isolation.

\textbf{Sparkov.} We use the dataset's standard public training and test split~\cite{sparkovdataset}. Fourteen engineered numeric features cover amount, location, time, demographics, and merchant category. The benchmark searches LightGBM and XGBoost with FLAML. The five seeds change both coreset selection and model search, which explains part of the observed dispersion.

\textbf{CriteoPrivateAds.} We process the full ${\sim}100$M-row corpus and follow the temporal protocol intended by the dataset authors~\cite{criteoprivatead}: days 1--24 form the training split and days 25--30 form the test split. The conversion label is $\texttt{nb\_sales}>0$, with a $0.77\%$ positive rate. Identifier fields, all-null fields, and high-cardinality hash identifiers are excluded, leaving 51 features. Every method uses the same zero-shot FLAML LightGBM procedure and five random seeds. The full-data baseline is one reference run. The CC-BY-SA 4.0 dataset is used under the applicable license, and no dataset rows are included in the paper.

\textbf{Methods.} Random retains all positives and samples negatives uniformly. CCS uses ten quantile strata, uniform allocation, and a $1\%$ hard-example cutoff. CoreTab uses the shared datamap implementation. The production CRISP runs use score-proportional sampling within each stratum ($\gamma=1$), while Sparkov and CriteoPrivateAds use uniform within-stratum sampling ($\gamma=0$). All CRISP runs use $Q=10$, $\alpha=\beta=1$, $w_{\max}=20$, and a proxy model trained on up to one million negatives plus all positives. CriteoPrivateAds runs the same CCS, CoreTab, and importance-sampling implementations used by the other benchmarks.

\textbf{Metric and repetition.} Average Precision is the primary metric because each task is highly imbalanced. The production dataset uses three seeds. Sparkov and Criteo use five. We report means and sample standard deviations across runs. We do not interpret seed dispersion as a formal significance test.

\begin{table*}[t]
\centering
\caption{Experiment configurations used by the three evaluations.}
\label{tab:experiment_config}
\small
\setlength{\tabcolsep}{5pt}
\begin{tabular}{@{}lllll@{}}
\toprule
\textbf{Dataset} & \textbf{Target training} & \textbf{Proxy model} & \textbf{Within stratum} & \textbf{Coreset runs} \\
\midrule
Production & FLAML zero-shot LightGBM, XGBoost transfer & LightGBM, depth 6, 300 trees & Score proportional ($\gamma=1$) & 3 \\
Sparkov & FLAML search over LightGBM and XGBoost & LightGBM, depth 6, 300 trees & Uniform ($\gamma=0$) & 5 \\
CriteoPrivateAds & FLAML zero-shot LightGBM & LightGBM, depth 6, 300 trees & Uniform ($\gamma=0$) & 5 \\
\bottomrule
\end{tabular}
\end{table*}

\subsection{Production Dataset Results}
Table~\ref{tab:main} focuses on rates where the methods separate. Figure~\ref{fig:production_results} shows the corresponding results across all tested reduction rates. At $95\%$ negative reduction, CRISP retains $99.7\%$ of full-data AP$_{\mathrm{month}}$. Random retains $94.2\%$, and CCS retains $94.1\%$. The expected CRISP training set contains about $1.70$M rows, or $6.8\%$ of the original data. At $80\%$ negative reduction, Random and CCS use about $5.38$M rows. CRISP therefore trains on roughly one-third as many rows at the comparison point while retaining more AP. Table~\ref{tab:retention} summarizes retention at the two highest production reduction rates.

\begin{table}[t]
\centering
\caption{Production Dataset AP retained relative to the corresponding full-data model. Values are means over three runs.}
\label{tab:main}
\scriptsize
\setlength{\tabcolsep}{3pt}
\begin{tabular}{@{}llcccccc@{}}
\toprule
 & & \multicolumn{6}{c}{\textbf{Negative-class reduction (\%)}} \\
\cmidrule(lr){3-8}
\textbf{Method} & & 50 & 60 & 70 & 80 & 90 & 95 \\
\midrule
\multirow{2}{*}{Random}
 & AP$_{\mathrm{month}}$ & 100.2\% & 99.7\% & 99.1\% & 98.3\% & 96.1\% & 94.2\% \\
 & AP$_{\mathrm{week}}$ & 99.4\% & 98.6\% & 98.4\% & 97.6\% & 95.5\% & 93.6\% \\
\multirow{2}{*}{CCS}
 & AP$_{\mathrm{month}}$ & 99.5\% & 99.5\% & 98.7\% & 97.7\% & 95.8\% & 94.1\% \\
 & AP$_{\mathrm{week}}$ & 98.7\% & 98.8\% & 98.2\% & 97.0\% & 95.8\% & 93.6\% \\
\multirow{2}{*}{\textbf{CRISP}}
 & AP$_{\mathrm{month}}$ & \textbf{100.3\%} & \textbf{100.5\%} & \textbf{100.4\%} & \textbf{100.3\%} & \textbf{100.4\%} & \textbf{99.7\%} \\
 & AP$_{\mathrm{week}}$ & \textbf{100.0\%} & \textbf{100.1\%} & \textbf{99.6\%} & \textbf{99.6\%} & \textbf{99.4\%} & \textbf{99.2\%} \\
\bottomrule
\end{tabular}
\end{table}

\begin{figure*}[t]
\centering
\includegraphics[width=0.90\textwidth]{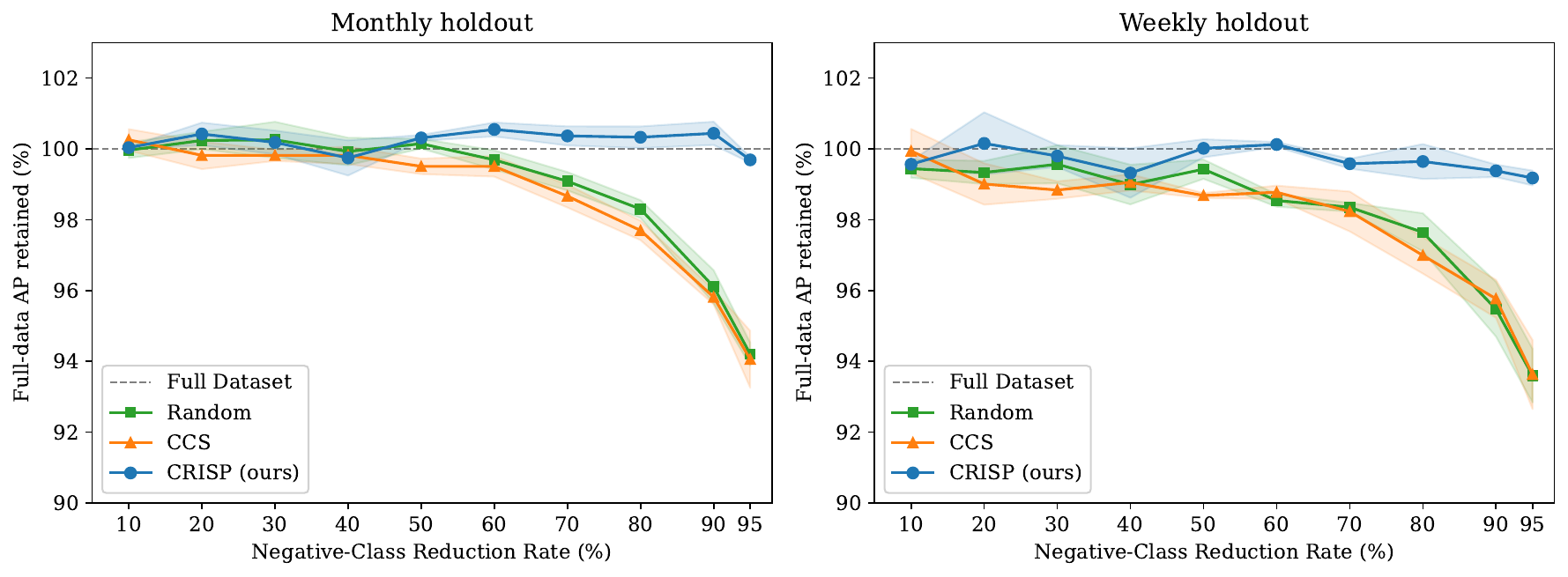}
\Description{Two line charts compare monthly and weekly full-data Average Precision retained on the Production Dataset across negative-class reduction rates for Random, CCS, and CRISP.}
\caption{Production Dataset AP retained relative to full-data AP. Lines show means over three runs, and shaded bands show one sample standard deviation.}
\label{fig:production_results}
\end{figure*}

\begin{table}[t]
\centering
\caption{Percentage of full-data Production Dataset AP retained at high negative-class reduction.}
\label{tab:retention}
\small
\begin{tabular}{@{}lcccc@{}}
\toprule
& \multicolumn{2}{c}{\textbf{90\%}} & \multicolumn{2}{c}{\textbf{95\%}} \\
\cmidrule(lr){2-3}\cmidrule(lr){4-5}
\textbf{Method} & AP$_{\mathrm{month}}$ & AP$_{\mathrm{week}}$ & AP$_{\mathrm{month}}$ & AP$_{\mathrm{week}}$ \\
\midrule
Random & 96.1\% & 95.5\% & 94.2\% & 93.6\% \\
CCS & 95.8\% & 95.8\% & 94.1\% & 93.6\% \\
\textbf{CRISP} & \textbf{100.4\%} & \textbf{99.4\%} & \textbf{99.7\%} & \textbf{99.2\%} \\
\bottomrule
\end{tabular}
\end{table}

\subsection{Ablation}
\label{sec:ablation}
Figure~\ref{fig:ablation} summarizes the ablation suite, which used a separate set of jobs from Table~\ref{tab:main}, so its reference values should not be compared as if they were the same runs. At $95\%$ negative reduction, the $Q=10$ setting retains $99.0\%$ of full-data monthly AP, while $Q=1$ retains $95.2\%$. Constant scores retain $94.6\%$. Removing IPW retains $97.4\%$, compared with $99.2\%$ for the matched clipped-IPW setting. Gradient-only scoring matches the combined score in this suite. Hessian-only scoring is worse because $p(1-p)$ falls when a negative receives a fraud score near one. Proxy-model depth and tree count change retained AP by at most $1.1$ percentage points in the collected runs.

\begin{figure*}[t]
\centering
\includegraphics[width=0.92\textwidth]{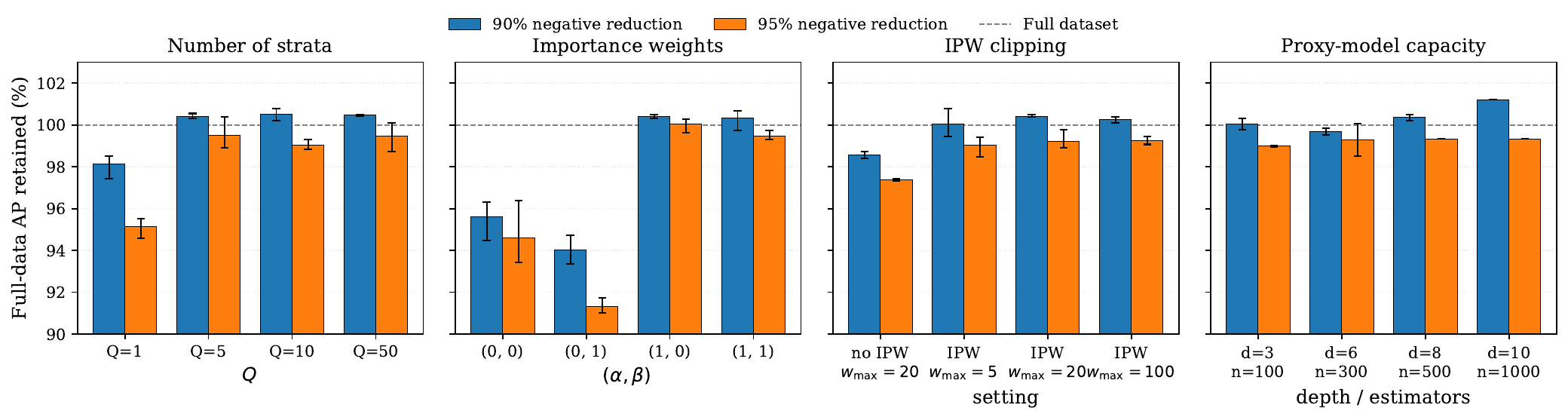}
\Description{Four grouped bar charts show full-data monthly Average Precision retained on the Production Dataset when varying the number of strata, score weights, inverse-propensity clipping, and proxy-model capacity.}
\caption{Production Dataset ablation suite at $90\%$ and $95\%$ negative-class reduction. Error bars span the observed minimum and maximum because some settings have only one or two completed seeds.}
\label{fig:ablation}
\end{figure*}

\subsection{Sparkov}
\label{sec:sparkov}
Sparkov is less conclusive than the production dataset. CRISP has the highest mean at $99.2\%$ and $99.4\%$ negative reduction, but it loses to Random and CCS at $99\%$. Its $99\%$ standard deviation is also the largest in the table because one run falls to $0.6647$. Random has the highest mean at $95\%$ because one run reaches $0.9629$. Appendix~\ref{app:robust_stats} reports medians and interquartile ranges so that these outliers remain visible without determining the entire interpretation.

\begin{table}[t]
\centering
\caption{Sparkov AP mean and sample standard deviation over five runs. Full-data mean AP is $0.8629$. Bold marks the best mean.}
\label{tab:sparkov}
\scriptsize
\setlength{\tabcolsep}{3pt}
\begin{tabular}{@{}lccccc@{}}
\toprule
& \multicolumn{5}{c}{\textbf{Negative-class reduction (\%)}} \\
\cmidrule(lr){2-6}
\textbf{Method} & 90 & 95 & 99 & 99.2 & 99.4 \\
\midrule
Random & .861$_{\pm.014}$ & \textbf{.883}$_{\pm.046}$ & \textbf{.833}$_{\pm.013}$ & .834$_{\pm.013}$ & .823$_{\pm.018}$ \\
CCS & .860$_{\pm.009}$ & .852$_{\pm.014}$ & .835$_{\pm.013}$ & .816$_{\pm.031}$ & .803$_{\pm.023}$ \\
CoreTab & \textbf{.863}$_{\pm.006}$ & .859$_{\pm.015}$ & .806$_{\pm.068}$ & .846$_{\pm.007}$ & .812$_{\pm.078}$ \\
\textbf{CRISP} & .862$_{\pm.013}$ & .874$_{\pm.004}$ & .817$_{\pm.089}$ & \textbf{.866}$_{\pm.009}$ & \textbf{.846}$_{\pm.021}$ \\
\bottomrule
\end{tabular}
\end{table}

\subsection{CriteoPrivateAds}
\label{sec:criteo}
The Criteo training split contains about $85$M rows with a $0.77\%$ positive rate. CRISP has the highest mean in each column of Table~\ref{tab:criteo}. At $90\%$ negative reduction, it reaches $0.1734$ AP on about $9.09$M target-training rows, or $10.7\%$ of the full split, and retains $99.8\%$ of the full-data AP. At $99.4\%$, about $1.16$M rows remain, or $1.37\%$ of the full split. CRISP retains $94.2\%$ of full-data AP, compared with approximately $90\%$ for Random and CCS and $76.7\%$ for CoreTab. The lead over the strongest non-CRISP baseline grows from $0.0034$ AP at $90\%$ reduction to $0.0074$ at $99.4\%$. At every rate, this gap is at least eight times the larger reported run-to-run standard deviation, although this descriptive ratio is not a formal hypothesis test.

\begin{table}[t]
\centering
\caption{CriteoPrivateAds AP mean and sample standard deviation over five runs. The single full-data reference has AP $0.1737$.}
\label{tab:criteo}
\scriptsize
\setlength{\tabcolsep}{2.5pt}
\begin{tabular}{@{}lccccc@{}}
\toprule
& \multicolumn{5}{c}{\textbf{Negative-class reduction (\%)}} \\
\cmidrule(lr){2-6}
\textbf{Method} & 90 & 95 & 99 & 99.2 & 99.4 \\
\midrule
Random & .1700$_{\pm.0002}$ & .1677$_{\pm.0003}$ & .1595$_{\pm.0004}$ & .1581$_{\pm.0003}$ & .1562$_{\pm.0004}$ \\
CCS & .1698$_{\pm.0001}$ & .1677$_{\pm.0002}$ & .1597$_{\pm.0003}$ & .1584$_{\pm.0004}$ & .1563$_{\pm.0009}$ \\
CoreTab & .1453$_{\pm.0008}$ & .1437$_{\pm.0009}$ & .1372$_{\pm.0010}$ & .1343$_{\pm.0021}$ & .1333$_{\pm.0014}$ \\
\textbf{CRISP} & \textbf{.1734}$_{\pm.0004}$ & \textbf{.1722}$_{\pm.0001}$ & \textbf{.1666}$_{\pm.0004}$ & \textbf{.1657}$_{\pm.0003}$ & \textbf{.1637}$_{\pm.0005}$ \\
\bottomrule
\end{tabular}
\end{table}

\begin{figure*}[t]
\centering
\includegraphics[width=0.94\textwidth]{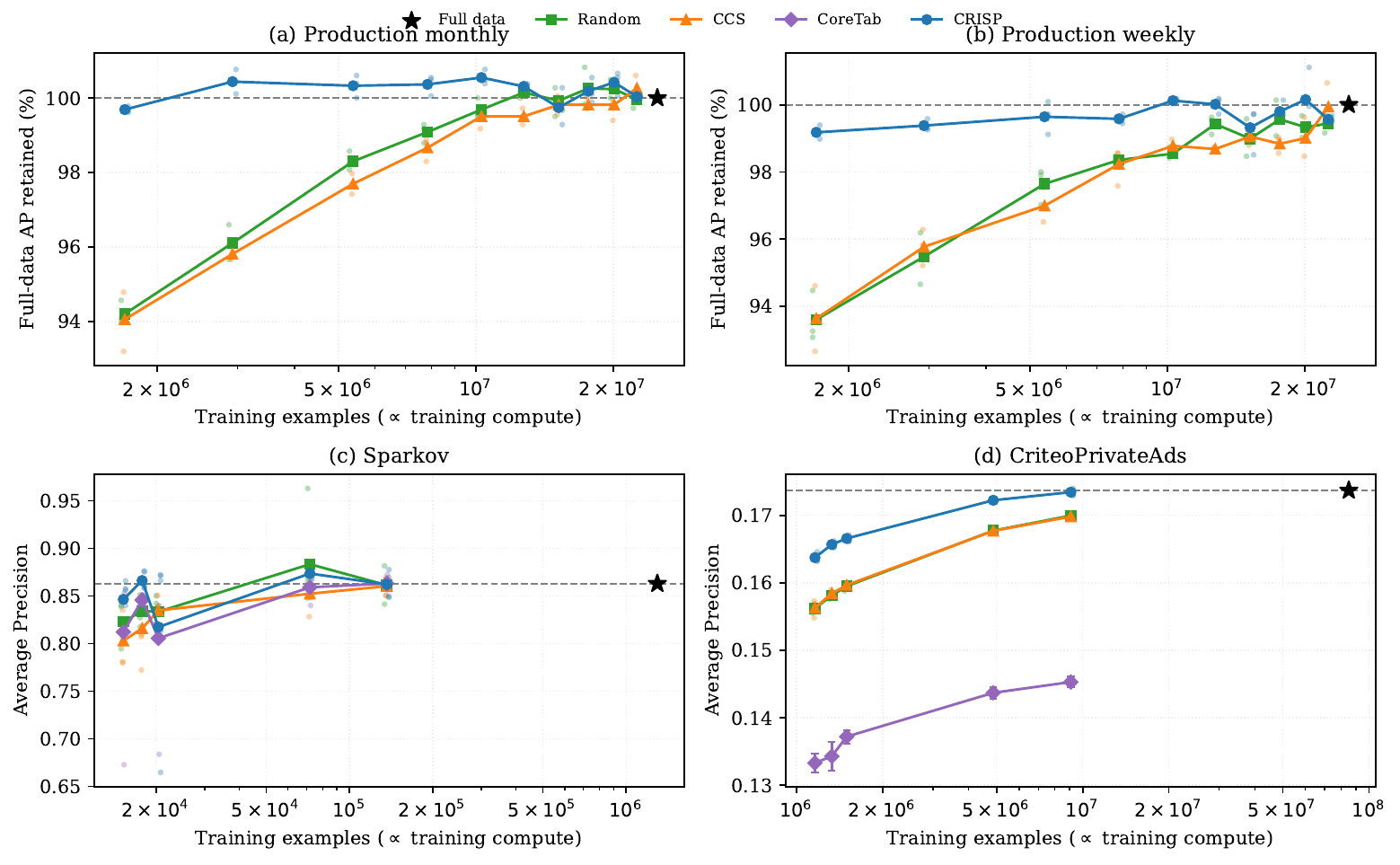}
\Description{Four panels plot relative or absolute Average Precision against approximate target-model training rows on a log scale for the Production Dataset monthly and weekly holds, Sparkov, and CriteoPrivateAds.}
\caption{Target-training row tradeoffs. Row counts are derived from reported class counts and the keep-all-positives rule. They approximate target-model training cost and exclude selection overhead.}
\label{fig:compute_tradeoffs}
\end{figure*}

\subsection{Target-Model Transfer}
Figure~\ref{fig:production_xgb} reports the existing Production Dataset XGBoost runs, which provide a second target model without new training. At $95\%$ negative reduction, CRISP retains $98.9\%$ of monthly full-data AP and $98.2\%$ of weekly full-data AP. Random retains $96.5\%$ and $94.7\%$, while CCS retains $96.2\%$ and $95.6\%$. This result supports transfer across two tree implementations, although it remains one dataset.

\subsection{Compute Interpretation}
Figure~\ref{fig:compute_tradeoffs} replots AP against derived target-training rows. The x-axis is not wall-clock time. CRISP adds a proxy-model fit and a distributed scoring pass before target training. Random does not. CCS also fits a proxy model, while CoreTab performs several prediction passes. The figure shows how target-model AP changes with training-set size. It does not establish end-to-end speedup.

\subsection{Cross-Dataset Analysis}
The clearest pattern is that allocation matters more as the negative budget shrinks. On the production dataset, Random and CCS remain close to the full-data result through moderate reduction, then decline after $60\%$. CRISP retains at least $99.7\%$ of monthly full-data AP through $95\%$ negative reduction. The value above $100\%$ retention at $90\%$ should not be read as evidence that pruning improves the population optimum. It is small relative to run-to-run variation and may reflect regularization, a changed class ratio, or ordinary training noise.

Criteo shows the same ordering with less seed variation. The CRISP margin over Random grows from $0.0034$ AP at $90\%$ negative reduction to $0.0075$ at $99.4\%$. This pattern is consistent with the allocation argument: uniform negative sampling becomes less likely to retain rare high-score regions as the budget contracts. The result is useful because Criteo differs from the production dataset in domain, feature construction, and target label, while using the same coreset implementations and configuration family.

Sparkov is a useful counterexample to a simple success narrative. Mean AP favors Random or CCS at $95\%$ and $99\%$, while CRISP leads at $99.2\%$ and $99.4\%$. The medians in Appendix~\ref{app:robust_stats} favor CRISP more consistently, which shows how strongly a single AutoML run can affect a five-seed mean. At $99.4\%$ negative reduction, the expected target-training set has only about $15{,}241$ rows. Model search and coreset randomness are both visible at that scale.

The ablation narrows the source of the production gain. Removing strata reduces coverage, while removing inverse weighting changes the effective training distribution. Proxy-model depth and tree count have much smaller effects in the tested range. These observations support the design choices, but they do not prove that the same allocation is optimal on every dataset. A fixed-hyperparameter study and end-to-end timing measurements remain necessary for a cleaner causal and systems comparison.

\section{Reproducibility}
\label{sec:reproducibility}

Sparkov and CriteoPrivateAds provide public benchmark data for evaluating the method outside the proprietary setting. CriteoPrivateAds is the primary public-data evaluation at scale. The paper documents the data splits, method and configuration family, reduction rates, random seeds, evaluation metrics, aggregate results, and public-benchmark per-run values in the appendix. No external artifact release is pledged as part of this submission.

The proprietary dataset cannot be released due to confidentiality restrictions. The production implementation, configurations, feature definitions, logs, run metadata, and underlying records are likewise not released. Absolute production AP values are withheld, and all reported production results are normalized by the corresponding full-data model. The public datasets are cited and used under their applicable licenses; their rows are not redistributed. Hardware details were not preserved consistently across historical runs, which is another reason the paper reports target-training rows rather than wall-clock time.

\section{Limitations and Ethics}
\label{sec:limitations}

\textbf{Method limitations.} Retaining every positive works only when the positive class is small enough to leave a useful negative budget. The production dataset uses score-proportional within-stratum sampling, while Sparkov and CriteoPrivateAds use uniform within-stratum sampling. Their shared evidence therefore concerns score strata, budget allocation, and weighting rather than one fixed within-stratum rule. Bernoulli sampling introduces small differences between target and realized coreset sizes, and weight clipping introduces the bias in Equation~\ref{eq:clipping_bias}.

\textbf{Evaluation limitations.} The row-count tradeoff is not an end-to-end timing result. It omits proxy-model fitting, distributed scoring, data movement, and calibration. Production target hyperparameters can vary because FLAML makes a new zero-shot suggestion for each selected dataset. This mirrors the existing training pipeline but weakens isolation of the selection method. The evaluation reports AP over three or five seeds and does not include a paired test, subgroup metrics, or a fixed-hyperparameter replication. The strongest claims are therefore about observed AP retention, not universal dominance or deployment cost.

\textbf{Ethics.} The production dataset is handled under institutional data-governance controls. The paper reports aggregate metrics and does not release transactions or user identifiers. The payment reversal fraud label reflects an operational outcome and should not be interpreted as evidence of a customer's intent. Direct identifiers are removed from the public datasets during preprocessing. Coreset selection can still alter subgroup representation and false-positive rates even when aggregate AP is stable. Before deployment, practitioners should compare subgroup recall, calibration, and manual-review burden against the full-data model. They should also check whether high-score strata overrepresent particular geographies, payment methods, or customer cohorts.

\section{Conclusion}
\label{sec:conclusion}

CRISP addresses a specific large-scale setting: binary tabular learning with a small positive class and a highly redundant negative class. A proxy GBDT places negatives into score quantiles, the mean score determines each stratum's budget, and inverse inclusion probabilities reduce the distribution shift introduced by unequal sampling. The procedure needs only a fixed-size proxy-model fit and distributed passes over the training table. It does not require pairwise gradients or predictions from many training checkpoints.

The strongest result comes from the production dataset. At $95\%$ negative-class reduction, CRISP retains $99.7\%$ of full-data monthly AP while using about $1.70$M of $25$M rows. CriteoPrivateAds shows the same mean ordering across five reduction rates and provides evidence outside fraud detection. Sparkov is less stable. Its outlier runs and changing method ranking show why a high mean at one rate is not enough to claim general superiority. The XGBoost experiment adds evidence that the selected production rows transfer across two tree implementations.

These experiments support importance-aware allocation when the negative budget is scarce. They do not establish a universal variance guarantee or an end-to-end wall-clock speedup. The next evaluation should use one within-stratum rule across all datasets, fixed target hyperparameters, actual selection and training times, and enough repeated runs for paired uncertainty estimates. Subgroup checks are also needed before the empirical package is complete.

\bibliographystyle{ACM-Reference-Format}
\bibliography{references}

\appendix
\section{Estimator Derivation}
\label{app:proofs}

\begin{proof}[Proof of Proposition~\ref{prop:unbiased}]
Condition on the realized design $\mathcal D$. The positive rows are retained deterministically. For each negative row, $\pi_i>0$ and
\begin{equation}
    \mathbb{E}\left[
    I_i\frac{\ell_i(\theta)}{\pi_i}\,\middle|\,\mathcal D
    \right]
    =\frac{\ell_i(\theta)}{\pi_i}
    \mathbb{E}[I_i\mid\mathcal D]
    =\ell_i(\theta).
\end{equation}
Summing over the negative rows and adding the deterministic positive sum gives
$\mathbb{E}[\widehat L(\theta)\mid\mathcal D]=L_{\mathcal V}(\theta)$.
This step uses linearity of expectation and does not require the indicators
to be independent. Unconditional unbiasedness follows from the law of total
expectation.
\end{proof}

\textbf{Clipping.} Conditional on $\mathcal D$,
$\mathbb{E}[I_i\widetilde w_i\ell_i\mid\mathcal D]
=\min(1,\pi_iw_{\max})\ell_i$. If $\pi_i\ge 1/w_{\max}$, the expected contribution remains $\ell_i$. Otherwise it is $\pi_i w_{\max}\ell_i$. Subtracting these terms from the full loss gives Equation~\ref{eq:clipping_bias}.

\textbf{Evaluated variants.} The production dataset samples within each stratum with probability proportional to score. Sparkov and CriteoPrivateAds use a constant Bernoulli fraction inside each stratum after the same importance-based allocation.

\section{Full Per-Run Experimental Results}
\label{app:fullresults}

Table~\ref{tab:production_runs} reports relative per-run Production Dataset results at the two highest reduction rates. Per-run Sparkov and CriteoPrivateAds tables follow.

\begin{table*}[t]
\centering
\caption{Per-run Production Dataset AP retained as a percentage of the corresponding full-data AP.}
\label{tab:production_runs}
\small
\setlength{\tabcolsep}{7pt}
\begin{tabular}{@{}clcccc@{}}
\toprule
& & \multicolumn{2}{c}{\textbf{Monthly AP retained}} & \multicolumn{2}{c}{\textbf{Weekly AP retained}} \\
\cmidrule(lr){3-4}\cmidrule(lr){5-6}
\textbf{Run} & \textbf{Method} & 90\% & 95\% & 90\% & 95\% \\
\midrule
\multirow{3}{*}{1}
 & Random & 96.6\% & 94.6\% & 96.2\% & 94.5\% \\
 & CCS & 95.7\% & 94.2\% & 95.8\% & 93.6\% \\
 & \textbf{CRISP} & \textbf{100.4\%} & \textbf{99.6\%} & \textbf{99.6\%} & \textbf{99.4\%} \\
\midrule
\multirow{3}{*}{2}
 & Random & 96.1\% & 94.1\% & 95.6\% & 93.3\% \\
 & CCS & 96.1\% & 93.2\% & 96.3\% & 92.7\% \\
 & \textbf{CRISP} & \textbf{100.8\%} & \textbf{99.7\%} & \textbf{99.3\%} & \textbf{99.2\%} \\
\midrule
\multirow{3}{*}{3}
 & Random & 95.7\% & 94.0\% & 94.7\% & 93.1\% \\
 & CCS & 95.7\% & 94.8\% & 95.2\% & 94.6\% \\
 & \textbf{CRISP} & \textbf{100.1\%} & \textbf{99.8\%} & \textbf{99.3\%} & \textbf{99.0\%} \\
\bottomrule
\end{tabular}
\end{table*}

\section{Full Per-Run Sparkov Results}
\label{app:sparkov_fullresults}

Table~\ref{tab:sparkov_runs} reports per-run AP on the Sparkov dataset. Full-data baselines are $0.8510$, $0.8594$, $0.8736$, $0.8563$, and $0.8742$ for seeds 1 through 5.

\begin{table*}[t]
\centering
\caption{Sparkov per-run AP across negative-class reduction rates.}
\label{tab:sparkov_runs}
\scriptsize
\setlength{\tabcolsep}{4pt}
\begin{tabular}{@{}clccccc@{}}
\toprule
& & \multicolumn{5}{c}{\textbf{Negative-Class Reduction (\%)}} \\
\cmidrule(lr){3-7}
\textbf{Run} & \textbf{Method} & 90 & 95 & 99 & 99.2 & 99.4 \\
\midrule
\multirow{4}{*}{1}
 & Random & .8815 & .8720 & .8195 & .8358 & .8213 \\
 & CCS & .8677 & .8567 & .8376 & .7722 & .7798 \\
 & CoreTab & .8614 & .8485 & .8368 & .8499 & .8458 \\
 & \textbf{CRISP} & .8691 & .8786 & .8659 & .8763 & .8561 \\
\midrule
\multirow{4}{*}{2}
 & Random & .8414 & .9629 & .8422 & .8272 & .8386 \\
 & CCS & .8501 & .8612 & .8506 & .8073 & .8063 \\
 & CoreTab & .8565 & .8616 & .8346 & .8500 & .6728 \\
 & \textbf{CRISP} & .8659 & .8765 & .6647 & .8752 & .8577 \\
\midrule
\multirow{4}{*}{3}
 & Random & .8614 & .8648 & .8501 & .8387 & .8402 \\
 & CCS & .8693 & .8620 & .8299 & .8094 & .7809 \\
 & CoreTab & .8618 & .8668 & .8325 & .8511 & .8443 \\
 & \textbf{CRISP} & .8496 & .8726 & .8722 & .8582 & .8399 \\
\midrule
\multirow{4}{*}{4}
 & Random & .8620 & .8692 & .8217 & .8177 & .8205 \\
 & CCS & .8508 & .8533 & .8395 & .8503 & .8351 \\
 & CoreTab & .8722 & .8784 & .6837 & .8439 & .8445 \\
 & \textbf{CRISP} & .8484 & .8681 & .8128 & .8617 & .8121 \\
\midrule
\multirow{4}{*}{5}
 & Random & .8601 & .8476 & .8333 & .8511 & .7945 \\
 & CCS & .8634 & .8283 & .8164 & .8403 & .8128 \\
 & CoreTab & .8633 & .8400 & .8400 & .8337 & .8531 \\
 & \textbf{CRISP} & .8776 & .8718 & .8711 & .8591 & .8659 \\
\bottomrule
\end{tabular}
\end{table*}

\section{Sparkov Robust Summaries}
\label{app:robust_stats}

Table~\ref{tab:sparkov_robust} reports medians and interquartile ranges across the same five runs as Table~\ref{tab:sparkov}. The mean and median rankings differ because Random, CoreTab, and CRISP each contain at least one large outlier.

\begin{table*}[t]
\centering
\caption{Sparkov AP median with interquartile range in brackets.}
\label{tab:sparkov_robust}
\scriptsize
\setlength{\tabcolsep}{3.5pt}
\begin{tabular}{@{}lccccc@{}}
\toprule
& \multicolumn{5}{c}{\textbf{Negative-Class Reduction (\%)}} \\
\cmidrule(lr){2-6}
\textbf{Method} & 90 & 95 & 99 & 99.2 & 99.4 \\
\midrule
Random & .861 [.860,.862] & .869 [.865,.872] & .833 [.822,.842] & .836 [.827,.839] & .821 [.821,.839] \\
CCS & .863 [.851,.868] & .857 [.853,.861] & .838 [.830,.840] & .809 [.807,.840] & .806 [.781,.813] \\
CoreTab & .862 [.861,.863] & .862 [.849,.867] & .835 [.833,.837] & .850 [.844,.850] & .845 [.844,.846] \\
CRISP & .866 [.850,.869] & .873 [.872,.877] & .866 [.813,.871] & .862 [.859,.875] & .856 [.840,.858] \\
\bottomrule
\end{tabular}
\end{table*}

\section{Production Dataset XGBoost Transfer Results}

\begin{figure*}[t]
\centering
\includegraphics[width=0.85\textwidth]{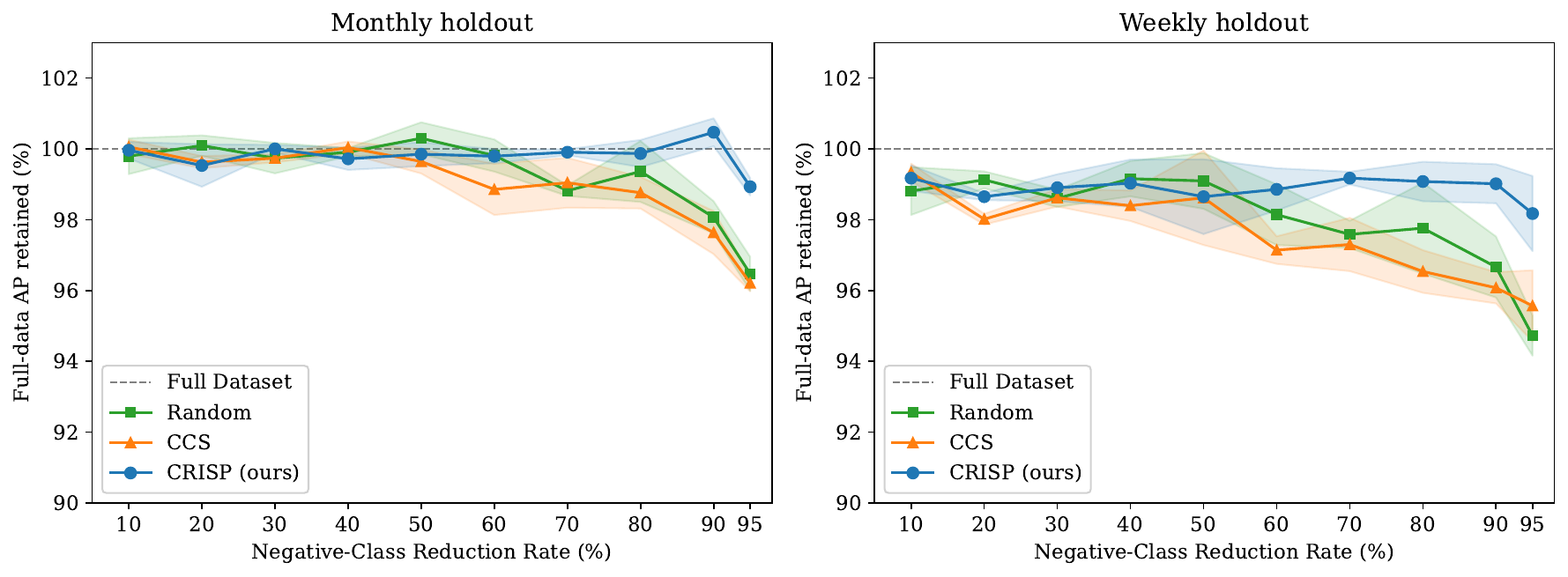}
\Description{Two line charts compare monthly and weekly full-data Average Precision retained on the Production Dataset for XGBoost targets across negative-class reduction rates.}
\caption{Production Dataset XGBoost target results over three runs, reported relative to full-data AP.}
\label{fig:production_xgb}
\end{figure*}

\section{Full Per-Run CriteoPrivateAds Results}
\label{app:criteo_fullresults}

Table~\ref{tab:criteo_runs} reports per-run AP on CriteoPrivateAds for seeds 42--46. The full-data baseline (AP $= 0.1737$) is a single reference run trained on all ${\sim}85$M training-split examples. CoreTab is reported only in aggregate in Table~\ref{tab:criteo} because its per-seed values were not retained in the exported summary.

\begin{table*}[t]
\centering
\caption{CriteoPrivateAds per-run AP across negative-class reduction rates.}
\label{tab:criteo_runs}
\small
\setlength{\tabcolsep}{6pt}
\begin{tabular}{@{}clccccc@{}}
\toprule
& & \multicolumn{5}{c}{\textbf{Negative-Class Reduction (\%)}} \\
\cmidrule(lr){3-7}
\textbf{Seed} & \textbf{Method} & 90 & 95 & 99 & 99.2 & 99.4 \\
\midrule
\multirow{3}{*}{42}
 & Random & .1701 & .1680 & .1600 & .1584 & .1564 \\
 & CCS & .1699 & .1680 & .1599 & .1586 & .1563 \\
 & \textbf{CRISP} & \textbf{.1735} & \textbf{.1723} & \textbf{.1662} & \textbf{.1652} & \textbf{.1632} \\
\midrule
\multirow{3}{*}{43}
 & Random & .1700 & .1676 & .1588 & .1576 & .1555 \\
 & CCS & .1699 & .1676 & .1597 & .1582 & .1561 \\
 & \textbf{CRISP} & \textbf{.1731} & \textbf{.1720} & \textbf{.1667} & \textbf{.1657} & \textbf{.1641} \\
\midrule
\multirow{3}{*}{44}
 & Random & .1702 & .1682 & .1595 & .1581 & .1563 \\
 & CCS & .1699 & .1677 & .1597 & .1578 & .1573 \\
 & \textbf{CRISP} & \textbf{.1740} & \textbf{.1723} & \textbf{.1661} & \textbf{.1662} & \textbf{.1633} \\
\midrule
\multirow{3}{*}{45}
 & Random & .1697 & .1674 & .1597 & .1585 & .1564 \\
 & CCS & .1698 & .1675 & .1592 & .1590 & .1548 \\
 & \textbf{CRISP} & \textbf{.1737} & \textbf{.1722} & \textbf{.1668} & \textbf{.1659} & \textbf{.1646} \\
\midrule
\multirow{3}{*}{46}
 & Random & .1698 & .1675 & .1595 & .1580 & .1565 \\
 & CCS & .1696 & .1677 & .1598 & .1585 & .1569 \\
 & \textbf{CRISP} & \textbf{.1729} & \textbf{.1724} & \textbf{.1671} & \textbf{.1654} & \textbf{.1635} \\
\bottomrule
\end{tabular}
\end{table*}

\end{document}